\documentclass{article}

\usepackage[preprint]{neurips_2026}

\usepackage[utf8]{inputenc} 
\usepackage[T1]{fontenc}    
\usepackage{hyperref}       
\usepackage{url}            
\usepackage{booktabs}       
\usepackage{amsfonts}       
\usepackage{amsmath}
\usepackage{nicefrac}       
\usepackage{microtype}      
\usepackage{xcolor}         

\usepackage{graphicx}
\usepackage{pgfplots}
\pgfplotsset{width=6cm,
    every non boxed x axis/.append style={x axis line style=-},
    every non boxed y axis/.append style={y axis line style=-},
    every axis x label/.style={at={(current axis.right of origin)},anchor=west},
    every axis y label/.style={at={(current axis.above origin)},anchor=south}}

\definecolor{BURed}{HTML}{CC0000}
\definecolor{BUBlack}{HTML}{2D2926}
\definecolor{AI2Green}{HTML}{105257}
\definecolor{AI2Pink}{HTML}{F0529C}

\DeclareMathOperator*{\argmax}{argmax}
\DeclareMathOperator{\train}{train}
\DeclareMathOperator{\mrts}{MRTS}

\usepackage{amsmath}
\usepackage{amsthm}
\usepackage{amssymb}
\usepackage{url}
\usepackage{xcolor}
\usepackage{cleveref}
\usepackage{csquotes}
\usepackage{paralist}
\usepackage{thm-restate}

\usepackage{cancel}

\newtheorem{lemma}{Lemma}
\newtheorem{theorem}{Theorem}

\theoremstyle{definition}

\newcommand\dc[3]{\textsf{\color{#2} [#1: #3]}}

\newcommand\WM[1]{\dc{WM}{orange}{#1}}

\renewcommand\dc[3]{}  

\title{A Theory of Training Profit-Optimal LLMs}

\author{%
  Sophie Hao$^*$ \\
  Boston University \\
  Boston, MA, USA \\
  \texttt{uu@bu.edu} \\
  \And
  William Merrill\thanks{Equal contribution.} \\
  Allen Institute for AI \\
  Seattle, WA, USA \\
  \texttt{willm@allenai.org} \\
}

\begin{document}

\maketitle

\begin{abstract}
    Scaling large language models (LLMs) requires tremendous computational resources, and recent advances in AI have gone hand in hand with massive amounts of capital expenditure. While it is established that scaling up LLMs reliably increases model quality (quantified in terms of loss or downstream evaluations), it is unclear how these quality improvements translate to potential revenue, and whether revenue increases would offset costs of larger-scale training and inference. In this work, we develop an economic model for characterizing the rational behavior of an LLM training firm by combining scaling laws with microeconomic theory. Under our model, LLM quality can be increased with more parameters and training tokens, leading to more potential adoption by consumers, who each have a quality threshold for using the LLM. On the other hand, additional parameters and training tokens both incur additional costs. We analyze the profit maximization problem for this model under compute-bound and data-bound regimes. In the compute-bound regime, optimal model size and token budget track hardware efficiency $E$ (FLOPs/\$) at a near-linear rate; total training cost then scales sub-quadratically in $E$. Data efficiency improvements incentivize larger models and training expenditure. When we are limited to $D$ data, profit-optimal training expenditure scales as $D^2 / E$, i.e, increase with data and \emph{decreases} with hardware efficiency (as well as data efficiency). Finally, we analyze practical trends in training expenditure: current trends are consistent with our most permissive model variants in the compute-bound regime, but are not profit-optimal in the data-bound regime or assuming hardware advances will stall. Overall, our results provide a theory of profit-optimal LLM training, providing a foundation for engaging critically with industry statements and supporting long-term economic decision making.
\end{abstract}

\section{Introduction}

At the time of writing, tremendous capital expenditure has gone towards training LLMs.
At a high level, this bet is motivated by the empirical phenomenon of \emph{scaling laws}: making LLMs larger and training them on more data monotonically improves their quality, measured in terms of training loss \citep{kaplanScalingLawsNeural2020,hoffmann2022training}.
Higher LLM quality is associated with better performance on downstream tasks \citep{wei2022emergent}.
This suggests scaling up LLM training could lead to models that are useful for many potential users, and thus profitable for their trainers.

On the other hand, improving LLM quality by scaling them up makes them more expensive, both in terms of raw computation (``compute'', measured in FLOPs) and dollars. The compute required to train an LLM is proportional to both the parameter count $n$ and training data budget $d$, so scaling up both of these increases compute quadratically.
Inference compute scales with $n$ but not $d$.
As $n$ and $d$ have jumped orders of magnitude, the compute required for training and inference have also increased exponentially.
While quality improvements are monotonic in these variables, they are also \emph{diminishing}, leading to the question of whether doubling down on making training more expensive to chase vanishing quality improvements will make LLM training profitable in the long run.

\begin{figure}
    \centering
    \footnotesize
    \begin{tikzpicture}
        \begin{axis}[
            axis lines = left,
            xmajorticks = false,
            ymajorticks = false,
            x = .01cm,
            y = .00325cm,
            xlabel = {Hardware Efficiency\\(FLOPs/\$)},
            ylabel = {LLM Size\\(params)},
            xlabel style={align=center,at={(axis description cs:0.5,-0.1)},anchor=north},
            ylabel style={align=center,at={(axis description cs:-0.1,.5)},rotate=90,anchor=south},
            ymax = 615,
            xmax = 196,
            ymin = 0,
            xmin = 0
        ] 
        \addplot [thick, BURed] table [col sep=comma] {neurips-images/e_vs_n.csv}; \end{axis} 
    \end{tikzpicture}
    \begin{tikzpicture}
        \begin{axis}[
            axis lines = left,
            xmajorticks = false,
            ymajorticks = false,
            x = .01cm,
            y = 0.000173cm,
            xlabel = {Hardware Efficiency\\(FLOPs/\$)},
            ylabel = {Train Expenditure\\(\$)},
            xlabel style={align=center,at={(axis description cs:0.5,-0.1)},anchor=north},
            ylabel style={align=center,at={(axis description cs:-0.1,.5)},rotate=90,anchor=south},
            ymax = 11577,
            xmax = 196,
            ymin = 0,
            xmin = 0
        ] 
        \addplot [thick, BURed] table [col sep=comma] {neurips-images/e_vs_c.csv}; \end{axis} 
    \end{tikzpicture}
    \begin{tikzpicture}
        \begin{axis}[
            axis lines = left,
            xmajorticks = false,
            ymajorticks = false,
            x = .1cm,
            y = .0556cm,
            xlabel = {Parameter Efficiency\\(LLM quality/param)},
            ylabel = {LLM Size\\(params)},
            xlabel style={align=center,at={(axis description cs:0.5,-0.1)},anchor=north},
            ylabel style={align=center,at={(axis description cs:-0.1,.5)},rotate=90,anchor=south},
            ymax = 36,
            xmax = 20,
            ymin = 0,
            xmin = 0
        ] 
        \addplot [thick, BURed] table [col sep=comma] {neurips-images/a_vs_n.csv}; \end{axis} 
    \end{tikzpicture}
    \begin{tikzpicture}
        \begin{axis}[
            axis lines = left,
            xmajorticks = false,
            ymajorticks = false,
            x = .1cm,
            y = .000263cm,
            xlabel = {Parameter Efficiency\\(LLM quality/param)},
            ylabel = {Train Expenditure\\(\$)},
            xlabel style={align=center,at={(axis description cs:0.5,-0.1)},anchor=north},
            ylabel style={align=center,at={(axis description cs:-0.1,.5)},rotate=90,anchor=south},
            ymax = 7613,
            xmax = 20,
            ymin = 0,
            xmin = 0
        ] 
        \addplot [thick, BURed] table [col sep=comma] {neurips-images/a_vs_c.csv}; \end{axis} 
    \end{tikzpicture}
    \caption{
    Our model predicts that a profit-maximizing LLM firm scales its LLM training expenditures sublinearly with hardware efficiency and inversely with parameter efficiency. Plots are shown for $\gamma = -1$ (to be defined in \autoref{sec:consumer-behavior}).}
    \label{fig:plots}
\end{figure}
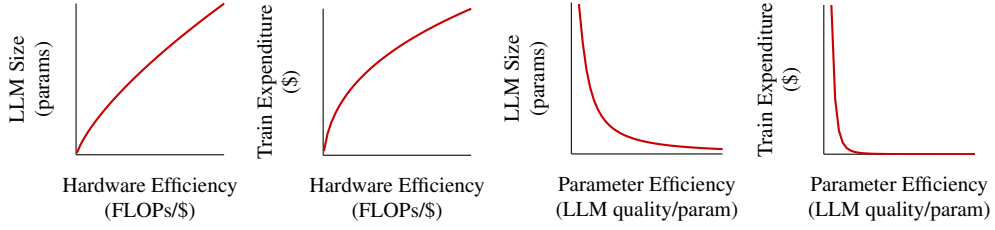
In this paper, we address this open question by developing a theory of \emph{profit-optimal} training behavior for LLM firms.
Under our model, the firm chooses their model size $n$ and training data budget $d$.
Larger choices of $n$ and $d$ produce higher-quality LLMs, which increases demand for the LLM, allowing the firm to charge more per inference token.
On the other hand, larger choices of $n$ and $d$ incur additional training and inference costs, suggesting that there is some  \emph{profit-optimal} $n^*, d^*$, i.e., a choice of these variables that maximizes profit.
We characterize $n^*, d^*$ as a function of exogenous variables like hardware efficiency (FLOPs/\$), the parameter and data efficiency of LLM training methods, and various natural constants.

Our first contribution is to formalize the profit maximization problem for an LLM firm with a monopoly on the market (\autoref{sec:formalization}).
LLM quality as well as training and inference costs can be measured relatively uncontroversially; the major challenge here is defining precisely how LLM quality improvements affect demand and token price.
We formalize this in a general way where each consumer has a minimum quality threshold that makes the LLM useful to them (e.g., it sufficiently solves all tasks relevant to their domain).
Then, the relationship between quality and demand reduces to a question about the distribution of this quality threshold across consumers.
We make the general assumption that it follows a power law $\propto 1 - q^{-\gamma}$, where the exponent $\gamma$ controls the degree to which inverse demand is diminishing in quality.
Following standard ideas in economics \citep{acemogluSimpleMacroeconomicsAI2025}, we assume inverse demand is diminishing or at most linear in quality, i.e., $\gamma < -1$.
This weak assumption suffices to prove our main results.

Having formalized the profit maximization problem, \autoref{sec:compute-bound} characterizes profit-optimal LLM training in the \emph{compute-bound} regime, i.e., the setting where the LLM firm is limited by training and inference costs but not by the amount of data available.
We find that optimal model size $n^*$ and data budget $d^*$ increase with hardware efficiency $E$ at a near-linear rate of $E^{1/(1+\alpha\gamma)}$, where $\alpha$ is the Chinchilla parameter exponent \citep{hoffmann2022training}.
thus, overall training compute $C^*_{\train}$ scales with $E$ at a rate of $E^{\frac{1-\alpha\gamma}{1+\alpha\gamma}}$.
We also consider the role of LLM training methodology: depending on the sign of $\gamma$, improvements in parameter efficiency can either increase or decrease $C^*_{\train}$.
In contrast, data efficiency improvements reduce $n^*$ but keep $C^*_{\train}$ fixed.

Additionally, in \autoref{sec:data-bound}, we characterize profit-optimal scaling in the data-bound regime where a fixed maximum data budget $D$ is prescribed.
Here, $n^*$ scales with $D$ at a near-linear rate, independent of $E$.
Notably, training expenditure $C^*_{\train}$ increases roughly quadratically with $D$, but \emph{decreases} with $E$.
Both $n^*$ and $C^*_{\train}$ improve with data efficiency advances but decrease with parameter efficiency advances in the data-bound regime.

Finally, we consider in \autoref{sec:empirics} how the characterization of profit-optimal LLM training under our model matches the empirical growth trends in training expenditure and related variables.
With $\gamma = 0$, which we take as a weak prior, we find that current training expenditure exceeds what is profit-optimal in the compute-bounded regime.
Solving for the value of $\gamma$ that would make current trends profit-optimal, we find $\hat \gamma \approx -0.77$, meaning that inverse demand is barely diminishing in LLM quality.
Thus, there is some version of our model where current trends are consistent with profit-optimal training behavior in the compute-bounded regime.

Overall, we extend the compute-optimal training framework \citep{hoffmann2022training} to model profit-optimal LLM training; we also comprehensively characterize profit-optimal training in the compute-bound and data-bound regime.
We hope our results can provide a rigorous framework for forecasting future developments in LLM training and critically engaging with industry trends; to this end, we include a comprehensive discussion of our results' implications, underlying assumptions, and rectification with other narratives on the profitability of LLM scaling (\autoref{sec:discussion}).

\section{Setup: LLM Firms and Profit Maximization} \label{sec:formalization}


In microeconomics, \textit{firms} are entities that take \textit{input factors} and produce an \textit{output good} that is sold to consumers. A pizzeria, for example, is a firm whose input factors are pizza ingredients, rent, and labor, and whose output good is pizza. The \textit{theory of the firm} aims to describe the behavior of firms in terms of the quantity of inputs they consume and outputs they produce, assuming that each firm maximizes profit within a market that may or may not be competitive.\footnote{See \citet{varianIntermediateMicroeconomicsModern2024} for an overview of the relevant microeconomic theory for this paper.}

In this section, we develop a microeconomic model that describes the behavior of a firm that trains an LLM and uses it to run an AI chatbot service. The firm's inputs consist of \textit{training data} and \textit{compute}, and its output consists of \textit{tokens} that are sold to consumers. Our goal is to make predictions about the relation between expectations of increased compute efficiency and the firm's investment into scaling.

Our model has the following characteristics. Because the focus of this paper is on LLM scaling and not on the effects of competition, we assume that the LLM firm operates with monopoly power. 
The LLM firm maximizes profit by deciding how many tokens to produce and sell, subject to technological constraints and consumer demand. Additionally, the LLM firm decides how much data and compute it will invest into training the LLM. A greater investment of input factors endows the LLM firm with a higher \textit{quality} LLM, which in turn increases demand for the LLM's tokens. 

\subsection{Consumer Behavior} \label{sec:consumer-behavior}

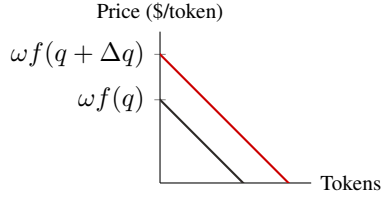
\begin{figure}
    \centering
    \begin{tikzpicture}
        \begin{axis}[
            axis lines = left,
            xmajorticks = false,
            x = .2cm,
            y = .2cm,
            xlabel = {\footnotesize Tokens},
            ylabel = {\footnotesize Price (\$/token)},
            ytick = {5.5, 8.5},
            yticklabels = {$\omega f(q)$, $\omega f(q + \Delta q)$},
            ymax = 10,
            xmax = 10,
            ymin = 0,
            xmin = 0
        ]
            \addplot[color=BURed, domain=-10:10, thick]{8.5 - x};
            \addplot[color=BUBlack, domain=-10:10, thick]{5.5 - x};
        \end{axis}
    \end{tikzpicture}
    \caption{Inverse demand functions for tokens generated by an LLM of quality $q$ (\textcolor{BUBlack}{\textbf{black}}) and $q + \Delta q$ (\textcolor{BURed}{\textbf{red}}), where $\Delta q > 0$. For any particular level of LLM quality $q$, demand for tokens is linear, with $\omega\ln(q)$ being the highest possible price that a token could be sold for. Training a better model increases the demand for tokens from that model.}
    \label{fig:demand-tokens}
\end{figure}

LLMs are general-purpose, open-ended AI models, which can be applied to a potentially unlimited range of \textit{tasks}. LLMs of higher \textit{quality}, as measured by inverse next-token prediction loss, have been shown to achieve better performance \citep{kaplanScalingLawsNeural2020,srivastava2023beyond} on a wider range of tasks \citep{brownLanguageModelsAre2020a,wei2022emergent}. Accordingly, AI chatbot services are typically priced by the inference token, with higher prices charged for tokens generated by a higher-quality LLM.

In the theory of the firm, consumer behavior is described by an \textit{inverse demand function} that predicts the unit price $p$ at which a good is sold from the quantity $t$ of the good that is sold. Following the \textit{law of demand}, we assume that $p$ decreases with respect to $t$.
Additionally, since higher-quality LLMs can be applied to a wider range of tasks, we assume that $p$ increases with respect to some measure of LLM quality $q$, which currently is left generic.
We capture both of these dependencies by proposing the following \textit{quasilinear} inverse demand function:
\[
p(t, q) = \omega f(q) - \delta t ,
\]
where $f(q)$ is some \emph{linking function} from quality to inverse demand that must be strictly monotonic and differentiable with respect to $q$.
Much of our analysis will apply to any choice of $f$ satisfying these properties, but three natural options are $f(q) = \ln(q)$, $f(q) = f_\gamma(q)$ with $\gamma > 0$, and $f(q) = f_\gamma(q)$ with $\gamma < 0$, where
\[
f_\gamma(q) = \frac{1}{\gamma} \left( 1 - q^{-\gamma} \right).
\]
All of these satisfy the axioms above and also yield diminishing returns to increasing model quality as long as $\gamma > -1$. Setting $\gamma = -1$ makes inverse demand linear in quality. Furthermore, we will see that they all can be motivated as special cases of the same general framework.

\begin{figure}
    \centering
    \footnotesize
    \begin{tabular}{c c c}
        \begin{tikzpicture}
            \begin{axis}[
                axis lines = left,
                xmajorticks = false,
                ymajorticks = false,
                x = .167cm,
                y = .167cm,
                xlabel = {\footnotesize $q$},
                ylabel = {\footnotesize $f_{-1}(q)$},
                ymax = 12,
                xmax = 12,
                ymin = 0,
                xmin = 0
            ]
                \addplot[domain=1:12, thick, BURed]{x - 1};
            \end{axis}
        \end{tikzpicture}
        & 
        \begin{tikzpicture}
            \begin{axis}[
                axis lines = left,
                xmajorticks = false,
                ymajorticks = false,
                x = .167cm,
                y = .167cm,
                xlabel = {\footnotesize $q$},
                ylabel = {\footnotesize $f_{0}(q)$},
                ymax = 12,
                xmax = 12,
                ymin = 0,
                xmin = 0
            ]
                \addplot[domain=1:12, thick, BURed]{ln(x)};
            \end{axis}
        \end{tikzpicture}
        & 
        \begin{tikzpicture}
            \begin{axis}[
                axis lines = left,
                xmajorticks = false,
                ymajorticks = false,
                x = .167cm,
                y = .167cm,
                xlabel = {\footnotesize $q$},
                ylabel = {\footnotesize $f_{1}(q)$},
                ymax = 12,
                xmax = 12,
                ymin = 0,
                xmin = 0
            ]
                \addplot[domain=0:10, thick, BURed]{1 - 1/x};
            \end{axis}
        \end{tikzpicture}
    \end{tabular}
    \caption{Inverse demand linking functions are parameterized by $\gamma$, which controls the degree to which the demand for tokens generated by LLM quality experiences diminishing returns to scale.}
\end{figure}
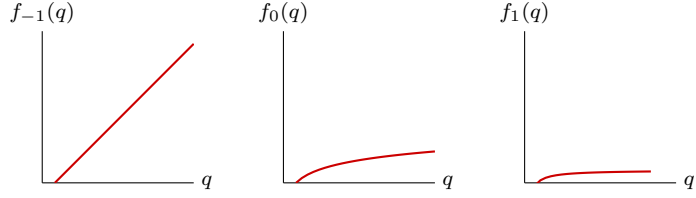
\paragraph{Derivation of Linking Functions.} We derive the general form of $f_\gamma(q)$ from the following assumptions, somewhat similar to the \textit{quantization model} framework of \citet{michaud2023quantization}:
\begin{enumerate}
    \item Each potential consumer of the LLM has a \emph{reservation quality} $q_* > 0$ such that they will pay for the LLM if and only if its quality $q \geq q_*$.
    \item This implies a \emph{density} over reservation qualities giving the rate at which tasks are unlocked as quality increases. This density follows a power law $1/q_*^{1+\gamma}$, for some $-1 < \gamma$.
    \item The maximum price that can be charged for an LLM token is proportional to number of consumers willing to buy that token; i.e., the number of consumers for whom $q \geq q_*$.
\end{enumerate}
Writing $p(t, q) = \omega f_\gamma(q) - \delta t$, we have
\[
f_\gamma(q) = \max_{t \geq 0} \frac{p(t, q)}{\omega} = \int_0^q \frac{1}{q_*^{1 + \gamma}} \, \mathrm{d}q_* .
\]
We show in \autoref{sec:linking-function} that this formula recovers $f_\gamma(q)$ in the way it was defined above.

\paragraph{Related Approaches.} Our inverse demand function is related to several prominent accounts of the relationship between quality and demand, as applied to LLMs. \citet{spenceMonopolyQualityRegulation1975} and \citet{mussaMonopolyProductQuality1978} assume that quality increases demand for a particular good or service. \citeauthor{michaud2023quantization}'s (\citeyear{michaud2023quantization}) \textit{quantization model}, as well as the economic models of \citet{acemogluSimpleMacroeconomicsAI2025} and \citet{bergemannMenuPricingLarge2026}, ground LLM quality in the number of tasks that the LLM can solve.


\subsection{LLM Scaling}
\begin{figure}
    \centering
    \footnotesize
    \begin{tabular}{c c c c}
        \textbf{Leontief} & \textbf{Chinchilla} 
        & \textbf{Perfect Substitutes} \\
        $\sigma = 0$ & $\sigma = .76$ 
        & $\sigma = \infty$\\
        \begin{tikzpicture}
            \begin{axis}[
                axis lines = left,
                xmajorticks = false,
                ymajorticks = false,
                x = .167cm,
                y = .167cm,
                xlabel = {\footnotesize $n$},
                ylabel = {\footnotesize $d$},
                ymax = 12,
                xmax = 12,
                ymin = 0,
                xmin = 0
            ]
                \draw[thick, BURed] (axis cs:5,12) -- (axis cs:5,5) -- (axis cs:12,5);
            \end{axis}
        \end{tikzpicture}
        & 
        \begin{tikzpicture}
            \begin{axis}[
                axis lines = left,
                xmajorticks = false,
                ymajorticks = false,
                x = .167cm,
                y = .167cm,
                xlabel = {\footnotesize $n$},
                ylabel = {\footnotesize $d$},
                ymax = 12,
                xmax = 12,
                ymin = 0,
                xmin = 0
            ]
                \addplot[domain=1:12, thick, BURed]{1 / (1.21037 - 1 / (x^.31205))^3.2046};
            \end{axis}
        \end{tikzpicture}
        & 
        \begin{tikzpicture}
            \begin{axis}[
                axis lines = left,
                xmajorticks = false,
                ymajorticks = false,
                x = .167cm,
                y = .167cm,
                xlabel = {\footnotesize $n$},
                ylabel = {\footnotesize $d$},
                ymax = 12,
                xmax = 12,
                ymin = 0,
                xmin = 0
            ]
                \addplot[domain=0:10, thick, BURed]{10 - x};
            \end{axis}
        \end{tikzpicture}
    \end{tabular}
    \caption[A scaling law's elasticity of substitution $\sigma$ measures the curvature of its \textit{iso-quality curves}.]{A scaling law's elasticity of substitution $\sigma$ measures the curvature of its \textit{iso-quality curves}, where $q(n, d)$ is constant. When $\sigma < 1$, we say that $n$ and $d$ are \textit{complements}, and quality is optimized when $n$ and $d$ are scaled together. When $\sigma \geq 1$, $n$ and $d$ are \textit{substitutes}, meaning they can be exchanged for one another without sacrificing quality.}
    \label{fig:elasticity}
\end{figure}
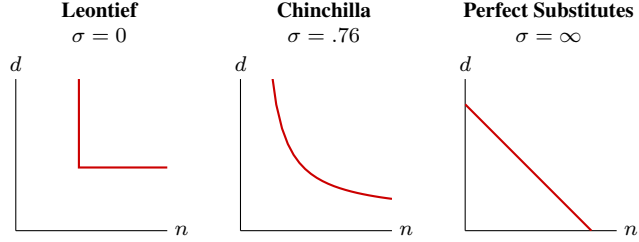


Research on LLM scaling has shown that an LLM's loss $\ell$ on next-token prediction is determined by the LLM's \textit{model size} $n$, in number of trainable parameters, and \textit{training data size} $d$, in tokens \citep{kaplanScalingLawsNeural2020,hoffmann2022training}, where higher-quality models generally have a lower loss. 
\citet{hoffmann2022training} in particular show that model size and training data size are \textit{complements}: scaling both $n$ and $d$ results in the greatest possible increase in LLM quality. To capture this relation, we propose the \textit{Leontief scaling law with diminishing returns to scale} for model quality as a function of $n$ and $d$:
\[
q(n, d) = \min\lbrace an^\alpha, bd^\beta \rbrace
\]
where the parameters $a > 0$ and $b > 0$ represent \textit{parameter efficiency} and \textit{data efficiency}, respectively, and $0 < \alpha \leq 1$ and $0 < \beta \leq 1$. 

The Leontief scaling law is based on descriptions of production technologies where input factors are \textit{perfect complements}; i.e., no increase in production is obtained unless all input factors are scaled simultaneously \citep{leontiefStructureAmericanEconomy1941}. The degree to which input factors must be scaled together is measured by a scaling law's \textit{elasticity of substitution}, $\sigma$. As shown in \autoref{fig:elasticity}, $\sigma$ measures the the degree to which $n$ and $d$ can be substituted with one another without sacrificing quality. $n$ and $d$ are \textit{substitutes} when $\sigma \geq 1$ and complements when $\sigma < 1$. 
In \autoref{sec:elasticity} we show that \citeauthor{hoffmann2022training}'s (\citeyear{hoffmann2022training}) \textit{Chinchilla scaling law} has an elasticity of substitution of $\sigma \approx .76$, which makes $n$ and $d$ complements. The Leontief scaling law simplifies our analysis by idealizing this relationship.

\subsection{Profit Maximization Problem}

The LLM firm faces a \textit{profit maximization problem} given by:
\[
n^*, d^*, t^* = \argmax_{n, d, t} \pi(n, d, t) \mathrel{\text{subject to}} n \geq 0, d \geq 0, t \geq 0
\]
where $\pi(n, d, t)$ is the profit earned by selling $t$ tokens generated by an LLM of size $n$ trained on $d$ tokens of data. Following microeconomic theory, we assume that the firm chooses values of $n$, $d$, and $t$ that maximize profit. The solution to the profit maximization problem therefore gives a complete description of the LLM firm's behavior.

The LLM firm's profit $\pi$ is given by its \textit{revenue} $R$ minus its \textit{cost} $C$. The LLM firm's revenue is the amount of money it earns by selling $t$ tokens generated by an LLM of quality $q$ given by the Leontief scaling law, at the price of \$$p$/token given by the inverse demand function.
\[
R(n, d, t) = p(t, q(n, d))t = \omega t \cdot f(q(n, d)) - \delta t^2
\]

The LLM firm's cost is that of purchasing $c_{\train} + c_{\inf}$ FLOPs of training and inference compute, respectively, at a price of \$$p_c$/FLOP. We assume without loss of generality that training data are free. The price of compute is determined by $p_c = 1/E$, where $E$ is the equilibrium \textit{hardware efficiency} supplied by the market. We follow \citet{kaplanScalingLawsNeural2020} in assuming that $c_{\train} = 6nd$ and $c_{\inf} = 2nt$,
which implies that the costs of training and inference in dollars are $C_{\train} = 6nd / E$ and $C_{\inf} = 2nt/E$, respectively.
This means the overall cost incurred by the LLM firm is
\[
C(n, d, t) = C_{\train} + C_{\inf} = \frac{6nd + 2nt}{E} .
\]
Finally, defining profit in terms of revenue and cost, we have
\begin{equation*}
    \pi(n, d, t) = R(n, d, t) - C(n, d, t) = \omega t \cdot f(q(n, d)) - \delta t^2 - \frac{6nd + 2nt}{E} .
\end{equation*}

\section{Behavior of Profit-Maximizing Firms}

We analyze the LLM firm's behavior by studying the local maxima of the profit function $\pi$. In this section, we solve for $t^*$ and $d^*$ in terms of $n$, reducing $\pi$ to one variable. The analysis of this section is agnostic to our choice of inverse demand linking function $f$.

To solve for $t$, we observe that $\pi$ is quadratic and concave-down in $t$. Thus, for any given value of $n$ and $d$, $t^*$ is given by the following \textit{first-order condition}:
\begin{equation}
0 = \frac{\partial \pi(n, d, t^*)}{\partial t} = \omega \cdot f(q(n, d)) - 2\delta t - \frac{2n}{E}. \tag{$t$} \label{eq:t}
\end{equation}
Solving (\ref{eq:t}) for $t$ gives us the following.
\begin{lemma} \label{lem:t}
    Consider any $f$ and $q$ and fix $n$ and $d$.
    Then, $\pi(n, d, t)$ has a local maximum at $t = t^*(n, d)$, where
    \[
    t^*(n, d) = \frac{1}{2\delta} \left( \omega \cdot f(q(n, d)) - \frac{2n}{E} \right).
    \]
\end{lemma}

By substituting $t = t^*(n, d)$, we express the profit function in terms of $n$ and $d$ only:
\[
\pi(n, d) = \pi(n, d, t^*(n, d))
= \frac{1}{4\delta}\left( \omega \cdot f(q(n, d)) - \frac{2n}{E} \right)^2 - \frac{6nd}{E} .
\]



Next, we show that under the Leontief scaling law, we can eliminate the variable $d$ in $\pi$ without loss of generality by assuming that $an^\alpha = bd^\beta$. This is because $n$ and $d$ are perfect complements under Leontief scaling: when $an^\alpha = bd^\beta$, $q(n, d)$ does not improve with further training.
We show here that the same holds of the LLM firm's profit, which means that any profit-optimal $n^*, d^*$ must also be \emph{quality-optimal}, in a sense analogous to Chinchilla compute optimality \citep{hoffmann2022training}.

\begin{lemma} \label{lem:quality-optimal-leontief}
    Let $q(n, d) = \min\lbrace an^\alpha, bd^\beta \rbrace$ and let $n, d \geq 0$.
    Then there exist $n^\prime \in [0, n]$ and $d^\prime \in [0, d]$ such that $a(n^\prime)^\alpha = b(d^\prime)^\beta$ and:
    \begin{enumerate}
        \item Quality is preserved, i.e., $q(n', d') = q(n, d)$.
        \item For all $t > 0$ we have $\pi(n', d', t) \geq \pi(n, d, t)$, with equality if and only if $(n^\prime, d^\prime) = (n, d)$.
    \end{enumerate}
    Thus, for $t > 0$, any choice of $n, d$ that maximizes profit $\pi$ must satisfy $a n^\alpha = b d^\beta$.
\end{lemma}

\begin{proof}
    By contradiction, assume we have some $n, d$ that maximizes profit such that $a n^\alpha > b d^\beta$. (The $a n^\alpha < b d^\beta$ case is analogous.)
    Choose $n' = \left(b/a\right)^{1/\alpha} d^{\beta/\alpha}$; observe that $n', d$ achieves the same quality as $n, d$ and thus the same revenue, due to the monotonicity of $f$.
    At the same time, it attains strictly lower cost because $n' < n$ and cost is monotonic in $n$, assuming $t > 0$. Thus, $\pi(n^\prime, d, t) > \pi(n, d, t)$ for any $t$.
\end{proof}
Solving $an^\alpha = bd^\beta$ for $d$, we obtain
\[
d = \left(\frac{a}{b}\right)^{1/\beta} n^{\alpha/\beta} = \rho n^{\alpha/\beta},
\]
where we define $\rho \triangleq (a/b)^{1/\beta}$ to be the Leontief scaling law's \emph{parameter-to-token efficiency factor}.
$\rho$ represents how much more efficient parameters are compared to additional training tokens for improving LLM quality.
Setting $d = \rho n^{\alpha/\beta}$, the single-variable version of the profit function is:
\[
\pi(n) = \pi(n, \rho n^{\alpha/\beta}) = \frac{1}{4\delta} \left( \omega f(an^\alpha) - \frac{2n}{E} \right)^2 - \frac{6\rho n^{1 + \alpha/\beta}}{E} .
\]

Finally, empirical research on LLM scaling has found that $\alpha \approx \beta$ \citep{hoffmann2022training,besirogluChinchillaScalingReplication2024,merrillOlmoHybridTheory2026};\footnote{\citet{hoffmann2022training} estimate $\alpha = .34$ and $\beta = .28$;  \citet{besirogluChinchillaScalingReplication2024} estimate $\alpha = .35$ and $\beta = .37$; \citet{merrillOlmoHybridTheory2026} estimate $\alpha = .25$ and $\beta = .21$, $\alpha = .23$ and $\beta = .22$, and $\alpha = .18$ and $\beta = .23$ for three families of LLMs.} in economic terms, LLM scaling is \textit{homogenous of degree $\alpha$}. Setting $\alpha = \beta$ simplifies $\pi(n)$ to
\[
\pi(n) = \pi(n, \rho n) = \frac{1}{4\delta} \left( \omega f(an^\alpha) - \frac{2n}{E} \right)^2 - \frac{6\rho n^2}{E}.
\]




\section{Profit-Optimal Scaling Based on Hardware and Algorithmic Efficiency} \label{sec:compute-bound}

We now investigate the \textit{comparative statics} of our model---how the LLM firm's behavior responds to changes in parameters that are exogenous to the firm's profit maximization problem. In particular, we study how the size of the LLM trained by the firm, $n^*$, as well as the LLM firm's total investment in training, $C^*_{\train}$, scale with $E$, $a$, and $\rho$.

In general, the LLM firm's profit maximization problem does not admit a tractable closed-form solution. In order to analyze the LLM firm's comparative statics, we derive asymptotic bounds on $n^*$ and $C^*_{\train}$ in terms of $E$, $a$, and $\rho$, treating separately the cases where the inverse demand function is given by $f_\gamma(q)$ where $\gamma \neq 0$ (\autoref{sec:full-quality}) and $f_0(q) = \ln(q)$ (\autoref{sec:log-quality}).

\subsection{Polynomial-Quasilinear Demand ($\gamma \neq 0$)}
\label{sec:full-quality}

We obtain the following characterization of the LLM firm's behavior when the inverse demand linking function $f_\gamma(q) = \frac{1}{\gamma} \left( 1 - q^{-\gamma} \right)$ for $\gamma \neq 0$.
We defer a proof to \autoref{app:poly-leontief}.

\begin{restatable}{theorem}{polyLeontief} \label{thm:compute-bound}
    Let $\alpha = \beta$.
    Suppose the inverse demand function is given by $f_\gamma$ for $\gamma \neq 0$. Then, assuming $\alpha\gamma \leq 1$ and $E > 1/(6\delta\rho)$, the solution to the profit maximization problem satisfies
    \[
    n^*
    = O\left( \left( \frac{E}{\rho a^\gamma} \right)^{1/(1 + \alpha\gamma)} \right)
    \quad \quad \quad
    d^*
    = O\left( \left(\frac{\rho^{\alpha\gamma} E}{a^\gamma}\right)^{1/(1+\alpha\gamma)} \right), 
    \]
    hence
    \[
    n^* = O\left( \frac{(Eb^{1/\alpha})^{1/(1 + \alpha\gamma)}}{a^{1/\alpha}} \right) \quad\quad\quad d^* = O\left(\left( \frac{E}{b^\gamma} \right)^{1/(1 + \alpha\gamma)}\right).
    \]
    
    As a result, the LLM firm's optimal training expenditure $C^*_{\train}$ is bounded as
    \[
    C^*_{\train}
    = O\left( a^{-\frac{2\gamma}{1 + \alpha\gamma}}\left(\frac{E}{\rho}\right)^{\frac{1-\alpha\gamma}{1 + \alpha\gamma}}   \right) = O\left(  \frac{(Eb^{1/\alpha})^{\frac{1 - \alpha\gamma}{1 + \alpha\gamma}}}{a^{1/\alpha}} \right).
    \]
\end{restatable}

\autoref{thm:compute-bound} says that LLM size $n^*$, training data budget $d^*$, and training expenditure $C^*_{\train}$ all increase with hardware efficiency $E$: the rate is superlinear for $\gamma < 0$, linear at $\gamma = 0$, and sublinear for $\gamma > 0$.
With $\alpha \approx 0.3$, we obtain an upper bound of $n^*, d^* \lesssim E^{1.43}$ and $C^*_{\train} \lesssim E^{1.86}$.
On the other hand, increasing parameter efficiency $a$ \emph{decreases} $n^*$ as well as $C^*_{\train}$.
In contrast, increasing data efficiency $b$ always increases $n^*$, increases $d^*$ if $\gamma > 0$, and increases $C^*_{\train}$ if $\gamma > 1/\alpha$.
\WM{Reformat this as a table with signs?}

\subsection{Log-Quasilinear Demand ($\gamma = 0$)}
\label{sec:log-quality}


We obtain the following characterization of the LLM firm's behavior when the inverse demand linking function $f_0(q) = \ln q$. We defer a proof to \autoref{app:log-leontief}.
\begin{restatable}{theorem}{logLeontief} \label{thm:log-characterization}
    Let $\alpha = \beta \leq 2$. Suppose the inverse demand function is given by $f_0$. Then, assuming $E > 1/(6\delta\rho)$, the solution to the profit maximization problem satisfies
    \begin{equation*}
        n^*
        = O \left( \frac{aE}{\rho} \right) = O\left( \frac{b^{1/\alpha} E}{a^{1/\alpha - 1}} \right)
        \quad \quad\quad
        d^* = O \left( aE \right) .
    \end{equation*}
    As a result, the optimal training expenditure $C^*_{\train}$ is bounded as
    \begin{equation*}
        C^*_{\train}
        =  O \left( \frac{a^2 E}{\rho} \right) = O\left( \frac{b^{1/\alpha} E}{a^{1/\alpha - 2}} \right) .
    \end{equation*}
\end{restatable}

\WM{Double check that $a$ isn't flipped in $n^*$, which would recover Thm1}
\WM{It could be that our theorem is right but a tighter bound that removes $a$ from $n$ is possible}

\autoref{thm:log-characterization} says that, with $\gamma = 0$, $n^*$, $d^*$, and $C^*_{\train}$ all scale linearly with hardware efficiency $E$.
Increased parameter efficiency $a$ increases optimal data budget $d^*$.
On the other hand, increasing $a$ decreases $n^*$ and $C^*_{\train}$ when $\alpha < 1/2$, which is satisfied in practice.
In contrast, data efficiency improvements incentivize larger models and overall training expenditure.

\section{Profit-Optimal Scaling in the Data-Bound Regime} \label{sec:data-bound}

We now turn to data-bound regime where the number of pretraining tokens is limited within $0 \leq d \leq D$.
In this regime, we can express our solution in terms of not just $A, E$, but also $D$.
Using the fact that $a(n^*)^\alpha = b(d^*)^\beta$~(\autoref{lem:quality-optimal-leontief}), we immediately obtain the following result when all available pretraining tokens are used:

\begin{theorem} \label{thm:leontief-data-bound}
    Enforce that $d = D$. 
    Then, the profit maximum is given by
    \begin{equation*}
        n^* = \left(\frac{D}{\rho}\right)^{\beta / \alpha} = \left(\frac{b}{a}\right)^{1/\alpha} D^{\beta/\alpha}.
    \end{equation*}
\end{theorem}
In other words, optimal model size depends on both algorithmic efficiency and data, but not compute efficiency.
For training expenditure, we get
\begin{equation*}
    C^*_{\train} = \frac{6n^*D}{E} = \Theta \left( \frac{D^{1 + \beta / \alpha}}{\rho^{\beta/\alpha} E} \right) = \Theta\left(\frac{D^{1 + \beta/\alpha}b^{1/\alpha}}{a^{1/\alpha}E}\right).
\end{equation*}
Thus, optimal training expenditure grows roughly quadratically with the amount of data available. It also grows with data efficiency but \emph{shrink} with parameter efficiency and hardware efficiency.

\section{Empirics: Real-World Training Expenditure Trends} \label{sec:empirics}

Our theoretical model makes for predictions for the way that advances in hardware and compute efficiency should should shift profit-optimal expenditure on training compute.
We now compare these to trends in practice using estimates for the annualized growth rates of these variables from the Epoch.ai dashboard\footnote{Taken from \url{https://epoch.ai/trends} on May 3, 2026.}.
They report training compute $\hat C_{\train} \propto 5^t$,
hardware efficiency $E \propto 1.37^t$,
and compute (``algorithmic'') efficiency $ab \propto 3^t$, where $t$ represents time in years.
We assume that compute efficiency improvements affect parameter and data efficiency equally, i.e., $a \propto b \propto \sqrt{3^t}$.

We can now compare the observed growth rate of $\hat C_{\train}$ to our upper bound on optimal training compute according to these efficiency measurements.
Parameterizing $C^*_{\train}$ in terms of these annualized growth rates, we have, for $\gamma \neq 0$,
\begin{align*}
    C^*_{\train}(\gamma)
    &\lesssim \left(\frac{E}{\rho}\right)^{\frac{1-\alpha\gamma}{1+\alpha\gamma}} \cdot a^{-\frac{2\gamma}{1+\alpha\gamma}} \\
    &= \left(\frac{1.37^t}{1}\right)^{\frac{1-\alpha\gamma}{1+\alpha\gamma}} \cdot \left(\sqrt{3^t}\right)^{-\frac{2\gamma}{1+\alpha\gamma}} \\
    &= \left( 1.37^{\frac{1-\alpha\gamma}{1+\alpha\gamma}} \cdot 3^{-\frac{\gamma}{1+\alpha\gamma}} \right)^t .
\end{align*}
We will assume $\alpha \approx \beta \approx 0.3$ and consider the trends under different choices of $\gamma$.
Profit-optimal training expenditure is largest under our model when $\gamma = -1$, where \autoref{thm:compute-bound} gives the bound
\begin{equation*}
    C^*_{\train}(-1) \lesssim \left( 1.37^{\frac{1+\alpha}{1-\alpha}} \cdot 3^{\frac{1}{1-\alpha}} \right)^t \approx 7.59^t .
\end{equation*}
Thus, the observed empirical growth rate is within the most permissive bounds of our model. However, as $\gamma$ increases, the exponents for $E$ and $a$ decrease.
With $\gamma = 0$, \autoref{thm:log-characterization} gives the bound
\begin{equation*}
    C^*_{\train}(0) \lesssim a^2 E \propto \left(3 \cdot 1.37\right)^t \approx 4.11^t .
\end{equation*}
Thus, assuming $\gamma = 0$, $\hat C_{\train}$ is growing too fast in practice compared to what our model predicts is profit-optimal.
Solving for the break-even point $\hat \gamma$ via
\begin{equation*}
    1.37^{\frac{1-\alpha\hat\gamma}{1+\alpha\hat\gamma}} \cdot 3^{-\frac{\hat\gamma}{1+\alpha\hat\gamma}} = 5 ,
\end{equation*}
we see that it is $\hat \gamma \approx -0.77$.
Thus, while the current rate of growth in training compute exceeds profit-optimal scaling under most choices of $\gamma$ under our model, there are some choices of $\gamma \approx -1$ where inverse demand is only slightly diminishing in quality that are consistent with current trends.

\section{Discussion: Is Current Training Expenditure Profit-Optimal?} \label{sec:discussion}

Before considering this question, we first summarize the qualitative findings under our model.
In the compute-bound regime, optimal model size, data budget, and training investment grow as hardware efficiency improves: the rates are at best subquadratic for $\gamma = -1$ but shrink if if inverse demand is more diminishing (as $\gamma$ increases).
While advances in data efficiency incentivize larger models and training expenditure, the role of parameter efficiency is less clear.

In the data-bound regime, more data being available incentivizes larger models at a near-linear rate (and thus compute increases near-quadratically).
In this regime, advances in data efficiency incentivize larger models and more training compute, whereas advances in parameter efficiency incentivize smaller models and less training compute.

We compare our model's predictions in the compute-bound case to the empirical trends in training compute vs.~hardware and compute efficiency.
Under the choice of $\gamma = 0$ (which we take as a reasonable default, though with significant uncertainty), we found that current growth rate in training compute \emph{exceeds} what would be profit-optimal.
Solving for the value of $\gamma$ that would make the current rate of compute growth profit-optimal, we found $\hat \gamma \approx -0.77$, which is within the range of what we considered possible in our model, though towards the lower end.
Further, in the data-bound case, our model predicts hardware efficiency improvements actually \emph{reduce} profit-optimal model size and training expenditure, so, if modern training runs are data-bound, the current growth rate in compute expenditure would unequivocally exceed the profit-optimal upper bound under our model.

\subsection{Rectification with Existing Narratives}

Significant ink has been spilled in popular discourse on the scaling of LLMs \citep[e.g., \emph{the Scaling Hypothesis};][]{branwenScalingHypothesis2022} and their potential profitability.
We therefore compare the predictions that our model makes for profit-optimal scaling against selected informal claims made on these themes.
First, we consider the following excerpt from \citeauthor{altmanThreeObservations2025}'s (\citeyear{altmanThreeObservations2025}) \emph{Three Observations} blog post:
\begin{displayquote}
    The cost to use a given level of AI falls about 10$\times$ every 12 months, and lower prices lead to much more use \dots\ [and] the socioeconomic value of linearly increasing intelligence is super-exponential in nature. A consequence of this is that we see no reason for exponentially increasing investment to stop in the near future.
\end{displayquote}
As the CEO of OpenAI,
\citeauthor{altmanThreeObservations2025} has a clear agenda here.
The claim that ``socioeconomic value'' is superexponential in ``intelligence'' does not seem well-defined, though perhaps it could be mapping onto inverse demand and quality in our framework.
Nevertheless, his claim that it could be profitable for training expenditure to grow exponentially with time is consistent with our model assuming hardware efficiency, data efficiency, or parameter efficiency (for certain values of $\gamma$) also continue to grow exponentially, and the amount of training data available does not become a constraint.


This begs the question: what happens in our model assuming hardware advances plateau?
In a recent blog post, \citet{dettmers2025agi} argues it is likely that the rate of growth in $E$ stalls due to physical limits, and this would limit the degree to which LLMs could be scaled up.
Putting aside the empirical question of whether this is true, we can analyze the implications of stalled hardware advances under our model by taking $E = O(1)$ w.r.t~time $t$.
Revisiting the analysis from \autoref{sec:empirics},
\begin{align*}
    C^*_{\train}(-1) &\lesssim \left( 3^{\frac{1}{1-\alpha}} \right)^t \approx 4.66^t \\
    C^*_{\train}(0) &\lesssim 3^t .
\end{align*}
That is, an exponential growth rate in $C^*_{\train}$ is still possible in certain regimes of $\gamma$ assuming modeling advances in parameter- and data-efficiency continue.
However, the current empirical growth rate in training expenditure $\hat C_{\train} \propto 5^t$ exceeds these upper bounds, even in the most permissive case of $\gamma = -1$.
Thus, assuming hardware advances will stagnate, our model predictions agree with \citeauthor{dettmers2025agi} that current rates of training expenditure would exceed what is profit-optimal.


\subsection{Key Assumptions of Our Model} \label{sec:assumptions}

A key assumption in our analysis is that the LLM firm has a monopoly on the market.
In the case where there are competing LLM firms, their profit-optimal behavior could change.
Extending our analysis to account for competition would be an interesting direction for future work.

Another simplification we have made is using Leontief scaling for quality rather than Chinchilla scaling, which is justified by the fact that $n, d$ are approximately complements under Chinchilla, i.e., both must increase together to avoid diminishing returns, as visualized in \autoref{fig:elasticity}.
It could still be interesting to extend the analysis to parameterize quality as the inverse of reducible loss under Chinchilla scaling, though overall we view Leontief as a quite reasonable approximation.

Finally, as discussed in \autoref{sec:consumer-behavior}, we assume the inverse demand is diminishing in model quality, i.e., $\gamma > -1$.
This is consistent with the law of diminishing returns in economics and similar to assumptions made in other analyses of the economic impacts of AI \citep{acemogluSimpleMacroeconomicsAI2025}; essentially, it follows if we view AI as ``normal technology'' \citep{narayanan2025ainormal}.
However, if one is convinced that quality increases in LLMs could be exceptionally transformative relative to quality improvements in other technologies, one might instead make the unconventional choice to model inverse demand as \emph{superlinear} in LLM quality, i.e., set $\gamma < -1$.
Similarly, our model does not account for the possibility of recursive self-improvement \citep{alignmentforum_rsi}, the idea that higher-quality models might accelerate the rate of improvement in hardware, parameter, and data efficiency.

\subsection{Future Work}

We believe that our upper bounds in \Cref{thm:compute-bound,thm:log-characterization} could be refined to get a tighter (potentially sublinear) scaling trend for $n^*$, $d^*$, and $C^*_{\train}$.


\clearpage
\begin{ack}
    WM thanks Andreas Ravichandran and SH thanks Frank Pinter and Patricia Sun for feedback.
\end{ack}

\bibliographystyle{acl_natbib}
\bibliography{sample}


\newpage
\appendix
\section{Derivation of Inverse-Demand Linking Functions}
\label{sec:linking-function}

Here we derive the form of the inverse demand linking functions by evaluating the integral from \autoref{sec:consumer-behavior}.
\begin{align*}
    f_0(q) &= \int_1^q \frac{1}{q_*} \, \mathrm{d}q_* = \ln(q) \\
    f_\gamma(q) &= \int_1^q \frac{1}{q_*^{1+\gamma}} \, \mathrm{d}q_* = -\frac{q_*^{-\gamma}}{\gamma} \bigg]^q_1 = \frac{1}{\gamma} \left( 1 - q^{-\gamma} \right)
\end{align*}
\section{Chinchilla Elasticity of Substitution}
\label{sec:elasticity}

Here we derive the elasticity of substitution of $\sigma \approx .76$ for \citeauthor{hoffmann2022training}'s (\citeyear{hoffmann2022training}) \textit{Chinchilla scaling law}.

\subsection{Chinchilla Scaling Law}
According to \citet{hoffmann2022training}, the next-token prediction loss $\ell$ of an LLM is predicted by
\[
\ell(n, d) = \ell_* + \frac{a}{n^\alpha} + \frac{b}{d^\beta}
\]
where the \textit{irreducible loss} $\ell_*$ is the best possible loss that can be achieved by an LLM. Since LLMs with lower loss are generally considered to be of higher quality, we assume that the quality of an LLM is inversely proportional to the amount by which its loss exceeds $\ell_*$. Therefore, the Chinchilla scaling law is given by
\[
q(n, d) = \frac{1}{\ell(n, d) - \ell_*} = \frac{1}{\frac{a}{n^\alpha} + \frac{b}{d^\beta}}.
\]

\subsection{Elasticity of Substitution}
When $q(n, d)$ is differentiable, the elasticity of substitution of $q$ is given by the formula
\[
\sigma = \frac{\mathrm{d} \ln(d/n)}{\mathrm{d} \ln(\mrts)}
\]
where
\[
\mrts = \frac{\partial q(n, d)/\partial n}{\partial q(n, d)/\partial d}
\]
is $q$'s \textit{marginal rate of technical substitution} (MRTS). 

The differential $\mathrm{d}\ln(d/n)$ is calculated using implicit differentiation along the curve where $q(n, d)$ is constant:
\[
0 = \mathrm{d}\left(\frac{a}{n^\alpha} + \frac{b}{d^\beta}\right) = -a\alpha n^{-\alpha - 1}\mathrm{d}n - b\beta d^{-\beta - 1}\mathrm{d}d \implies \mathrm{d}d = -\frac{a\alpha d^{\beta + 1}}{b\beta n^{\alpha + 1}} \mathrm{d}n,
\]
hence
\[
\mathrm{d}\ln\left(\frac{d}{n}\right) = \mathrm{d}(\ln(d) - \ln(n)) = \frac{\mathrm{d}d}{d} - \frac{\mathrm{d}n}{n} = \left( -\frac{a\alpha d^\beta}{b\beta n^{\alpha + 1}} - \frac{1}{n} \right)\mathrm{d}n.
\]

The MRTS is calculated as follows:
\begin{align*}
    \frac{\partial q(n, d)}{\partial n} &= \frac{a\alpha}{n^{1 + \alpha}(an^{-\alpha} + bd^{-\beta})^2} \\
    \frac{\partial q(n, d)}{\partial d} &= \frac{b\beta}{d^{1 + \beta}(an^{-\alpha} + bd^{-\beta})^2} \\
    \mrts &= \frac{a\alpha d^{1 + \beta}}{b\beta n^{1 + \alpha}},
\end{align*}
hence
\begin{align*}
    \mathrm{d}\ln(\mrts) &= \mathrm{d}\ln(a\alpha d^{1 + \beta}) - \mathrm{d}\ln(b\beta n^{1 + \alpha}) \\
    &= \frac{\beta + 1}{d} \mathrm{d}d - \frac{\alpha + 1}{n} \mathrm{d}n \\
    &= -\left(\frac{a\alpha(\beta + 1)d^\beta}{b\beta n^{\alpha + 1}} + \frac{\alpha + 1}{n} \right)\mathrm{d}n,
\end{align*}
and therefore
\[
\sigma = \frac{\mathrm{d} \ln(d/n)}{\mathrm{d} \ln(\mrts)} = \frac{\left( -\frac{a\alpha d^\beta}{b\beta n^{\alpha + 1}} - \frac{1}{n} \right)\mathrm{d}n}{-\left(\frac{a\alpha(\beta + 1)d^\beta}{b\beta n^{\alpha + 1}} + \frac{\alpha + 1}{n} \right)\mathrm{d}n} = \frac {a\alpha d^{\beta} + b\beta n^{\alpha}} {a\alpha (\beta + 1) d^{\beta} + (\alpha + 1) b\beta n^{\alpha}}.
\]
When $\alpha = \beta$, the above simplifies to
\[
\sigma = \frac{1}{1 + \alpha}.
\]
\citet{hoffmann2022training} report $\alpha = .3392$ and $\beta = .2849$; taking the average of these values, $\alpha = .31205$, yields $\sigma \approx .7622$.

\section{Proof of Poly-Quasilinear Demand Result} \label{app:poly-leontief}

\polyLeontief*

\begin{proof}
    The derivative of the profit function is given by
    \[
    \pi'(n) = \frac{\alpha\omega^2n^{-\alpha\gamma-1}}{2a^\gamma\gamma\delta} -\frac{\alpha\omega^2n^{-2\alpha\gamma-1}}{2a^{2\gamma}\gamma\delta} + \left( \frac{1}{\gamma} - \alpha \right)\frac{\omega n^{-\alpha\gamma}}{a^\gamma E\delta} + \left( \frac{1}{E\delta} - 6\rho\right)\frac{2n}{E} - \frac{\omega}{E\gamma\delta}.
    \]
    From this, we derive
    \begin{equation}
    0 = \underbrace{\frac{\alpha\omega {n^*}^{-\alpha\gamma-2}}{2a^\gamma\delta}}_{T_1}
        \underbrace{\mathrel{-}\frac{\alpha\omega {n^*}^{-2\alpha\gamma-2}}{2a^{2\gamma}\delta}}_{T_2}
        + \underbrace{\left( 1 - \alpha\gamma \right)\frac{{n^*}^{-\alpha\gamma-1}}{a^\gamma E\delta}}_{T_3}
        + \underbrace{\left( \frac{1}{E\delta} - 6\rho\right)\frac{2\gamma}{\omega E}}_{T_4}
        \underbrace{\mathrel{-}\frac{{n^*}^{-1}}{E\delta}}_{T_5} \label{eqn:poly-foc}
    \end{equation}
    by multiplying both sides of the first-order condition $\pi'(n^*) = 0$ by $\gamma/(n^*\omega)$.

    To obtain the upper bound, we observe that $T_2, T_5 < 0$. Thus, subtracting $T_2 + T_5$ from the right-hand side of Equation (\ref{eqn:poly-foc}) gives us:
    \begin{align*}
    0 &\leq 
        \frac{\alpha\omega {n^*}^{-\alpha\gamma-2}}{2a^\gamma\delta}
        + \left( 1 - \alpha\gamma \right)\frac{{n^*}^{-\alpha\gamma-1}}{a^\gamma E\delta}
        + \left( \frac{1}{E\delta} - 6\rho\right)\frac{2\gamma}{\omega E} \\
    &\leq
        \frac{\alpha\omega {n^*}^{-\alpha\gamma-1}}{2a^\gamma\delta}
        + \left( 1 - \alpha\gamma \right)\frac{{n^*}^{-\alpha\gamma-1}}{a^\gamma E\delta}
        + \left( \frac{1}{E\delta} - 6\rho\right)\frac{2\gamma}{\omega E} \\
    &=
        \left( \frac{\alpha\omega}{2a^\gamma\delta}
        + \frac{1 - \alpha\gamma}{a^\gamma E\delta} \right) {n^*}^{-(\alpha\gamma+1)}
        + \left( \frac{1}{E\delta} - 6\rho\right)\frac{2\gamma}{\omega E} .
    \end{align*}
    Solving for $n^*$, we have
    \begin{align*}
            -\left( \frac{1}{E\delta} - 6\rho\right)\frac{2\gamma}{\omega E}
            \left( \frac{\alpha\omega}{2a^\gamma\delta}
                + \frac{1 - \alpha\gamma}{a^\gamma E\delta} \right)^{-1}
        &\leq {n^*}^{-(\alpha\gamma+1)} \\
        \implies {n^*}^{\alpha\gamma+1}
        &\leq
            \frac{\omega E}{2\gamma \delta a^\gamma}
            \left( 6\rho - \frac{1}{E\delta} \right)^{-1}
            \left( \frac{\alpha\omega}{2} + \frac{1 - \alpha\gamma}{E} \right) \\
        &=
            \frac{\omega E}{2\gamma \delta a^\gamma}
            \cdot \frac{\alpha\omega E + 2 - 2\alpha\gamma}{2E\left( 6\rho - \frac{1}{E\delta} \right)} = O \left( \frac{E}{\rho a^{\gamma}} \right).
    \end{align*}
    Therefore we obtain
    \begin{equation*}
        n^*
        = O\left( \left(\frac{E}{\rho a^\gamma}\right)^{1/(\alpha\gamma+1)} \right) .
    \end{equation*}
    We can use \autoref{lem:quality-optimal-leontief} to bound $d^*$ as 
    \begin{equation*}
        d^*
        = \rho n^*
        = O\left( \left(\frac{\rho^{\alpha\gamma} E}{a^\gamma}\right)^{1/(\alpha\gamma+1)} \right) .
    \end{equation*}
    Putting it all together, we get
    \begin{align*}
        C^*_{\train}
        &= \frac{6n^*d^*}{E} \\
        &= O\left( \left(\frac{E}{\rho a^\gamma}\right)^{1/(\alpha\gamma+1)} \left(\frac{\rho^{\alpha\gamma} E}{a^\gamma}\right)^{1/(\alpha\gamma+1)} \right) \\
        &= O\left( a^{-\frac{2\gamma}{1 + \alpha\gamma}}\left(\frac{E}{\rho}\right)^{\frac{1-\alpha\gamma}{1 + \alpha\gamma}}   \right) \\
        &= O\left( a^{-1/\alpha} \left( \frac{E}{b^{1/\alpha}} \right)^{\frac{1 - \alpha\gamma}{1 + \alpha\gamma}}\right). \qedhere
    \end{align*}
\end{proof}

\section{Proof of Log-Quasilinear Demand Result} \label{app:log-leontief}
Throughout this section, we assume that $n^*$ is determined by the first-order condition $\pi'(n^*) = 0$. To simplify notation, let $F \triangleq (f \circ q)(n)$ and $F' \triangleq (f \circ q)'(n)$, where $q(n) = an^\alpha$.

The derivative of the profit function is
\begin{equation*}
    \pi'(n)
    = \frac{\omega^2}{2 \delta}F F' - \frac{\omega}{\delta E}(F + F'n) + \frac{2}{E} \left( \frac{1}{\delta E} - 6\rho\right) n .
\end{equation*}
\logLeontief*

\begin{proof}
Since $\gamma = 0$, we have $F = \alpha \ln n + \ln a$ and $F' = \alpha / n$.
We have the first-order conditions
\begin{align*}
    0 = \pi'(n) &=
    \frac{\alpha\omega^2}{2\delta} \frac{\ln(an^\alpha)}{n}
    - \frac{\omega}{\delta E} (\ln(an^\alpha) + \alpha) 
    + \frac{2}{E} \left( \frac{1}{\delta E} - 6 \rho \right) n \\
    \iff 0 &= \frac{\alpha\omega^2}{2\delta} \frac{\ln(an^\alpha)}{n^2}
    - \frac{\omega}{\delta E} \frac{\ln(an^\alpha) + \alpha}{n}
    + \frac{2}{E} \left( \frac{1}{\delta E} - 6 \rho \right).
\end{align*}
For $x>0$, $\frac{\ln(ax)}{x}$ is maximized when $x = \frac{e}{a}$ at $\frac{a}{e}$.
Thus, letting $x = n^\alpha$, we bound $\frac{\ln(an^\alpha)}{n^\alpha} \leq \frac{a}{e}$ in the first term:
\begin{equation*}
    0 \leq \frac{\alpha\omega^2 a}{2\delta en^{2 - \alpha}}
    - \frac{\omega}{\delta E} \frac{\ln(an^\alpha) + \alpha}{n}
    + \frac{2}{E} \left( \frac{1}{\delta E} - 6 \rho \right).
\end{equation*}
Next, since $n \geq 1$ and $\alpha \leq 2$, we get:
\begin{align*}
    0 &\leq
        \frac{\alpha\omega^2 a}{2\delta e n}
        - \frac{\omega(\ln(a) + \alpha)}{\delta E n}
        + \frac{2}{E} \left( \frac{1}{\delta E} - 6 \rho \right) \\
    &=
        \frac{1}{n}\left(\frac{\alpha\omega^2 a}{2\delta e}
        - \frac{\omega(\ln(a) + \alpha)}{\delta E}\right)
        + \frac{2}{E} \left( \frac{1}{\delta E} - 6 \rho \right) \\
   \implies -\frac{2}{E} \left( \frac{1}{\delta E} - 6 \rho \right) n &\leq \frac{\alpha\omega^2 a}{2\delta e}
        - \frac{\omega(\ln(a) + \alpha)}{\delta E} \\
   \left( \frac{12\rho}{E} - \frac{2}{\delta E^2} \right) n &\leq \frac{\alpha\omega^2 a}{2\delta e}
        - \frac{\omega(\ln(a) + \alpha)}{\delta E} \\
   \left( 6\delta\rho - \frac{1}{E} \right) n &\leq \frac{\alpha\omega^2 aE}{4 e}
        - \frac{\omega(\ln(a) + \alpha)}{2} .
\end{align*}
Since $E > 1/(6\delta\rho)$ by assumption,
we can divide both sides to get
\begin{align*}
    n^*
    &\leq \frac{\alpha\omega^2}{4e\left(6\delta\rho - \frac{1}{E}\right)} aE - \frac{\omega}{2(6\delta\rho - \frac{1}{E})}(\ln(a) + \alpha) .
\end{align*}
Absorbing constants, we have the following asymptotics for large $E$:
\begin{equation*}
    n^* = O\left( \frac{aE}{\rho} \right) .
\end{equation*}
By \autoref{lem:quality-optimal-leontief} and since $\alpha = \beta$, we can characterize $d^*$ as
\begin{equation*}
    d^* = \rho n^* = O \left( a E \right) .
\end{equation*}
Putting it all together, the optimal training compute allocation scales comparably to data:
\begin{equation*}
    C^*_{\train}
    = \frac{6n^*d^*}{E}
    = O\left( \frac{a^2 E}{\rho} \right) = O\left( \frac{b^{1/\alpha}E}{a^{1/\alpha - 2}}\right) . \qedhere
\end{equation*}
\end{proof}



\end{document}